\documentclass[letterpaper, 10 pt, conference]{ieeeconf}
\IEEEoverridecommandlockouts
\usepackage{cite}
\usepackage{amsmath,amssymb,amsfonts}
\DeclareMathAlphabet{\mathpzc}{OT1}{pzc}{m}{it}
\usepackage{algorithmic}
\usepackage{graphicx}
\usepackage{textcomp}
\usepackage{xcolor}
\usepackage{changes}
\usepackage{algorithm}
\usepackage{subcaption} 
\usepackage{float}

\usepackage{titlesec}
\titlespacing*{\subsection}{0pt}{\baselineskip}{\baselineskip}

\def\BibTeX{{\rm B\kern-.05em{\sc i\kern-.025em b}\kern-.08em
    T\kern-.1667em\lower.7ex\hbox{E}\kern-.125emX}}

\newtheorem{lemma}{Lemma}
\newtheorem{theorem}{Theorem}
\newtheorem{definition}{Definition}

\newtheorem{proposition}{Proposition}

\title{\LARGE \bf
Generalized Advantage Estimation for Distributional Policy Gradients
}
\author{
Shahil Shaik, Jonathon M. Smereka, and Yue Wang
\thanks{Shahil Shaik and Yue Wang are with the Mechanical Engineering Department, Clemson University; Jonathon M. Smereka is with the Ground Vehicle Systems Center U.S. Army CCDC. This work was supported by the Automotive Research Center (ARC) at the University of Michigan, Ann Arbor, under
Cooperative Agreement W56HZV-24-2-0001 with the US Army DEVCOM Ground Vehicle Systems Center (GVSC).
DISTRIBUTION A. Approved for public release; distribution unlimited. OPSEC $\#$9027
}%
}

\begin{document}
\maketitle




\begin{abstract}
Generalized Advantage Estimation (GAE) has been used to mitigate the computational complexity of reinforcement learning (RL) by 
employing an exponentially weighted estimation of the advantage function to reduce the variance in policy gradient estimates. Despite its effectiveness, GAE is not designed to handle value distributions integral to distributional RL, which can capture the inherent stochasticity in systems and is hence more robust to system noises. To address this gap, we propose a novel approach that utilizes the optimal transport theory to introduce a Wasserstein-like directional metric, which measures both the distance and the directional discrepancies between probability distributions. Using the exponentially weighted estimation, we leverage this Wasserstein-like directional metric to derive distributional GAE (DGAE). Similar to traditional GAE, our proposed DGAE provides a low-variance advantage estimate with controlled bias, making it well-suited for policy gradient algorithms that rely on advantage estimation for policy updates. We integrated DGAE into three different policy gradient methods. Algorithms were evaluated across various OpenAI Gym environments and compared with the baselines with traditional GAE to assess the performance.
\end{abstract}

\begin{keywords}
Generalized Advantage Estimation, Optimal Transport, Wasserstein-like Directional Metric, Exponentially-weighted Estimation, Policy Gradient, Distributional Reinforcement Learning 
\end{keywords}

\section{Introduction}

Reinforcement Learning (RL) has shown  significant promise in sequential decision-making, where agents interact with an environment to learn optimal decisions by maximizing expected returns. This approach has successfully addressed complex problems such as robotics and autonomous systems characterized by continuous state-action spaces. Policy gradient methods, which directly parameterize and optimize policies, excel in these settings by supporting continuous actions, stochasticity, and stable convergence \cite{Lillicrap}. 

A central challenge in policy gradient methods is balancing bias and variance in gradient estimation. Stochastic gradient ascent algorithms utilize a noisy yet unbiased estimate of the gradient of the expected return; however, the variance of this estimate increases unfavorably with the time horizon \cite{Schulman}, \cite{Sutton2}. 
To address this, actor-critic algorithms use a value function to reduce the variance of the gradient estimate, albeit at the cost of introducing some bias \cite{Sutton},\cite{Mnih}. 
Generalized Advantage Estimation (GAE) mitigates this trade-off by proposing a class of policy gradient estimators that reduce the variance of policy gradient estimates while maintaining a tolerable level of bias~\cite{Schulman}. However, this method relies on scalar value functions and cannot be directly applied to value distributions, a key aspect of distributional RL. 

Unlike traditional RL, distributional RL \cite{Bellemare} models complete return distributions instead of expected values, capturing environmental randomness and enabling robust policy learning in high-noise, dynamic settings. One of the key tools in distributional RL is the Wasserstein distance, a metric originally derived from optimal transport (OT) \cite{Bellemare,Santambrogio}. The Wasserstein distance quantifies the distance between two probability distributions. 
Its continuous and almost everywhere differentiable nature makes it particularly valuable for gradient-based learning, as these properties ensure meaningful learning curves even when dealing with distributions that have non-overlapping supports \cite{Arjovsky}. 
In particular, distributional RL algorithms leverage the Wasserstein distance to learn the value distributions by minimizing the distance between the predicted return distribution and its Bellman target \cite{Bellemare,Li,Dabney}. However, the Wasserstein metric lacks the ability to define the superiority between the distributions being compared. Hence, this approach cannot directly estimate the advantage function—a critical component for guiding policy updates in policy gradient methods, as it quantifies the relative benefit of a specific action compared to the policy's default action. We will need a metric to measure the distance and superiority between distributions to define the advantage function.

Our work bridges the aforementioned gaps by introducing a Wasserstein-like directional metric capable of measuring both the distance and superiority between distributions. Leveraging this novel metric, we define the distributional temporal difference (TD) error and propose an $n$-step advantage estimator that quantifies the discrepancy between the $n$-step return distribution and the baseline value distribution. Building on this, we introduce distributional GAE (DGAE), which computes an exponentially weighted average of the $n$-step advantages. This approach optimizes the bias-variance tradeoff by combining shorter horizon returns (lower bias, emphasizing immediate rewards) with longer horizon returns (lower variance, incorporating richer distributional information). As a result, we derive a low-bias, low-variance advantage estimator tailored for distributional policy gradient algorithms, enabling more stable and efficient optimization.


The rest of the paper is organized as follows. Section \ref{sec:preliminaries} introduces the key preliminaries and background information necessary to understand the proposed approach. Section \ref{sec:main_section} introduces the proposed Wasserstein-like directional metric and develops the DGAE based on this directional metric. This section also discusses integrating our approach with baseline policy gradient algorithms and value distribution learning. Section \ref{sec:results} describes the simulation setup and experimental results. Finally, Section \ref{sec:conclusion} concludes the paper with a summary of contributions.

\section{Preliminaries}\label{sec:preliminaries}

In this paper, we will use {$\mathbb{E}_{(\cdot)}$} to represent the expectation with respect to a random variable. 
We consider an  Markov Decision Process (MDP) with continuous state and action spaces represented by a tuple $(\mathcal{S}, \mathcal{A}, r, P, \gamma)$ where 
    $\mathcal{S}$ is the continuous state space, $\mathcal{A}$ is the continuous action space, $r: \mathcal{S} \times \mathcal{A} \rightarrow \left[r_{\min},r_{\max}\right]$ 
    is the reward function, mapping state-action pairs $(s_t, a_t), s_t\in \mathcal{S}, a_t\in\mathcal{A},$ to a bounded set of rewards $\left[r_{\min},r_{\max}\right]$, $P(s_{t+1}|s_t,a_t): \mathcal{S} \times \mathcal{A} \rightarrow \mathcal{P}_s$ 
     is the unknown probability density that governs the transition dynamics with $\mathcal{P}_s$ representing the state probability density set, 
     and $\gamma \in (0,1)$ is the discount factor that balances immediate and future rewards~\cite{Haarnoja}.
    We will use $\pi$ to denote a stochastic policy $\pi:\mathcal{S}\rightarrow \mathcal{P}_a$, 
    where $\mathcal{P}_a$ is the action probability distribution set. 



\subsection{Policy Gradients}

Policy gradient methods aim to directly optimize the policy to maximize the expected total return, i.e., by optimizing the objective $J(\pi_{\theta}) {= \mathbb{E}_{\pi_{\theta}}\left[\sum_{k=0}^{\infty} \gamma^k r_k \right]}$, where $\theta$ represents the policy’s parameters~\cite{Sutton2}. There are several different related forms of the objective function gradient used to update the policy parameters, and it is commonly expressed as 
\begin{equation}\label{eq:gradient_obj_fcn}
\nabla_\theta J(\pi_\theta) = \mathbb{E}_{\pi_\theta}\left[\sum_{k=0}^{\infty} \psi_k \nabla_\theta \log \pi_\theta(a_k | s_k)\right]
\end{equation}
Here, $\psi_k$ is some signal that can guide the policy training, for example, the discounted return $\sum_{k=0}^{\infty} \gamma^k r(s_k,a_k)$, the discounted state-action value function $Q^{\gamma, \pi_\theta}(s_t,a_t)$, the advantage function {$\mathpzc{A}^{\gamma, \pi_\theta}(s_t,a_t)$}, etc. The advantage function $\mathpzc{A}^{\gamma,\pi_\theta}(s_t,a_t) {\triangleq} Q^{\gamma,\pi_\theta}(s_t,a_t)-V^{\gamma,\pi_\theta}(s_t)$, if known, yields the lowest possible variance among these choices \cite{Schulman}. 


It is important to note that the true value functions $Q^{\gamma,\pi_\theta}, V^{\gamma,\pi_\theta}$ are usually unknown and are approximated {using} neural networks. Hence, the true advantage function $\mathpzc{A}^{\gamma,\pi_\theta}$ is also unknown and needs to be approximated. The advantage estimators will be represented as $\hat{\mathpzc{A}}^{(\cdot)}$ in this paper.

\subsection{Generalized Advantage Estimator (GAE)}
A common challenge in policy gradient methods is the issue of high variance in the policy gradient estimation, which usually results in unstable learning and slow convergence. GAE was proposed to mitigate this issue by reducing the variance in gradient estimation while introducing a controllable amount of bias \cite{Schulman}. It modifies the traditional advantage estimation by incorporating a temporal smoothing mechanism, 
utilizing the {exponentially} weighted average of the $k$-step estimators. Mathematically, {GAE} is given by
\begin{equation}\label{eq:GAE}  
{\hat{\mathpzc{A}}^{\gamma,\lambda}_{\text{GAE}}(s_t,a_t)} = \sum_{k=0}^{\infty}(\gamma\lambda)^k \delta(s_{t+k}, a_{t+k})
\end{equation}
where 
\begin{equation}
    \delta(s_t,a_t) = r_t(s_t,a_t) + \gamma V(s_{t+1}) - V(s_t)
\end{equation} 
is the TD error, and $\lambda \in (0,1)$ is a control parameter. The parameters $\gamma$ and $\lambda$ contribute towards bias-variance tradeoff \cite{Schulman}. This low-bias, low-variance advantage estimator $\hat{\mathpzc{A}}^{\gamma,\lambda}$ can be used to approximate the advantage in Eq. (\ref{eq:gradient_obj_fcn}) to guide the policy updates.

\subsection{Optimal Transport Theory}

A fundamental problem while working with probability distributions is defining a metric to measure the discrepancy between distributions. 
According to the optimal transport theory, the optimal transport cost between two distributions can be computed by solving the Kantorovich problem (KP) defined as follows:
\begin{definition}
    Given probability distributions $\mu\in{\mathcal{P}_U}$ and $\nu\in{\mathcal{P}_V}$ of random variables $U,V$, respectively, and a cost function $c:U\times V\rightarrow \mathbb{R}$, the {Kantorovitch problem (KP)} is given by:
    \begin{equation}
        \inf{\left\{K(\Gamma)\triangleq\int_{U\times V}c{(U,V)}d\Gamma : \Gamma\in\Pi(\mu,\nu)\right\}}
    \end{equation}
    where the set of joint distributions $\Pi(\mu,\nu)$ is the so-called transport plans.
\end{definition}

The KP problem has a unique optimal solution when the cost function $c(U,V)$ is strictly convex. However, if the cost function is only convex, 
the uniqueness of the optimal solution cannot be guaranteed anymore~\cite{Santambrogio}. 

For a random variable $U$, we denote the CDF by $F_U(u)\triangleq P\{U\leq u\}$, and its inverse CDF by $F_U^{-1}(q)\triangleq\inf\{u:F_U(u)\geq q\}$, where $q\in[0,1]$ represents quantiles. Proposition 2.17 and Remark 2.11 in \cite{Santambrogio} states that the optimal transport cost of the KP problem is as follows: 
\begin{proposition}\label{def:OT_cost}
    Let $h:\mathbb{R}\rightarrow\mathbb{R}$ be a strictly convex function, and let $F_U$ and $G_V$ be the CDFs of random variables $U$ and $V$, respectively. The KP for optimal transport between $F_U$ and $G_V$, with cost function $c(U,V)=h(U-V)$, admits the following optimal solution:
    \begin{equation}
        \begin{aligned}
            d^{OT}(F_U,G_V) &\triangleq \inf_{U,V}~ h(U - V) \\
            &= \int_0^1 h(F^{-1}_U(q) - G^{-1}_V(q)) \, dq
        \end{aligned}
    \end{equation}
\end{proposition}

The Wasserstein metric is a popular special case of Proposition \ref{def:OT_cost} where the convex cost function is chosen as $h(U-V)=\|U-V\|_p$, where $\|\cdot\|_p$ represents the $p$-norm. Hence, the $p$-Wasserstein distance $d_p$ can be given as \cite{Bellemare}:
\begin{equation}
    \begin{aligned}
        d_p &= \inf_{U,V}||U-V||_p
        = \left(\int_0^1|F_U^{-1}(q)-G_V^{-1}(q)|^p dq\right)^{1/p}
    \end{aligned}
\end{equation}

\subsection{Distributional Reinforcement Learning}

In traditional RL, the value function is typically defined as the expected total return, where the expectation is taken over both the policy and the transition dynamics. However, this approach only captures the mean of the return, failing to account for the inherent stochasticity in most MDPs, making the system more vulnerable to noise. To address this limitation, distributional RL seeks to model the entire distribution of returns. Specifically, the value distribution is defined as a mapping $G: \mathcal{S} \to \mathcal{G}$, where $\mathcal{G}$ represents the space of possible return distributions.
\begin{eqnarray}
    {G(s_t) \triangleq \left[\sum_{k=0}^{\infty} \gamma^k r(s_{t+k}, a_{t+k}) \Big{\vert} s_t\right]}
\end{eqnarray}
where {$a_k \sim {\pi}(\cdot \mid s_k)$, $s_{k+1} \sim P(\cdot \mid s_k, a_k)$, and $s_0 = s_t$.} Comparing with the traditional value function, it is clear that: 
\begin{equation}
    V(s_t) = \mathbb{E}_{{{\pi}, P}}[G(s_t)]
\end{equation}


The Bellman operator for the value distribution $G(s_t)$, denoted by $\mathcal{T}:\mathcal{G}\rightarrow\mathcal{G}$, is mathematically defined as~\cite{Li}:
\begin{equation}\label{eq:G_dis_bellman}
    \mathcal{T} G(s_t) \triangleq r(s_t,A_t) + \gamma G(S_{t+1})
\end{equation}
where $A_t \sim \pi(\cdot|s_t)$, $S_{t+1}\sim P(\cdot|s_t,A_t)$. Finally, the supremum of the $p$-Wasserstein metric between value distribution and the Bellman-updated distribution is:
\begin{equation}
    \bar{d_p}(G(s_t),\mathcal{T}G(s_t)) \triangleq \sup_{s\in\mathcal{S}}d_p(G(s_t),\mathcal{T}G(s_{t}))
\end{equation}
The Bellman operator for the state value distribution is a $\gamma$-contraction under this metric $\bar{d}_p$, which ensures the convergence of the value distribution during learning~\cite{Li}.







\section{Generalized Advantage Estimation for Distributional Policy Gradients} \label{sec:main_section}
\subsection{Wasserstein-like Directional Metric}


In the context of {optimal transport theory} \cite{Santambrogio}, Wasserstein distance can be interpreted as the amount of `mass' that needs to be transported between two distributions without specifying the direction of the transport. This is a limitation because it quantifies the distance between distributions but can not inherently provide a comparative assessment of which distribution is superior. However, in specific scenarios such as computing the advantage function, where we need to quantify how good or bad a particular action is compared to the policy's action, it is essential to determine the superiority between distributions along with the distance between them. This raises the question:  What defines the superiority of one distribution over another? As previously mentioned, the concept of comparing distributions is not straightforward and can be interpreted in different ways depending on the context. In the domain of distributional RL, when comparing value distributions, it is intuitive to consider a distribution with more mass concentrated toward higher values as superior. Thus, in this work, we define the CDF $F_U$ of a random variable $U$ as superior to the CDF $G_V$ of a random variable $V$ if the net mass transfer required to transform $F_U$ into $G_V$ is in a negative direction, and vice versa.\footnote{In this work, we define superiority based on the net mass flow direction; however, future work will incorporate the uncertainty present in the distributions.} 

Now that we have defined the term ``superiority" in the context of our work and the need for a metric that can define this superiority between distributions, we are ready to define our proposed Wasserstein-like directional metric. In contrast to the traditional Wasserstein metric, which uses a convex cost function $h: \mathbb{R} \rightarrow \mathbb{R}^+$ given by $h(U-V) = \|U - V\|_p$, we propose to adopt a linear cost function $h: \mathbb{R}\times\mathbb{R} \rightarrow \mathbb{R}$ that satisfies the convexity property given as
\begin{equation} \label{eq:linear_convex_cost}
    h(x, y) \triangleq L(x - y)
\end{equation} 
where $L$ is a linear function. Under a linear cost function, every transport plan is guaranteed to achieve optimality, provided the underlying probability measures have compact support~\cite{Santambrogio}. 

Consequently, we define the Wasserstein-like directional metric $d$ as follows:
\begin{definition}\label{def:Wasserstein_like_metric}
Let $h:\mathbb{R}\rightarrow\mathbb{R}$ be a linear convex function given by $h(x-y) = L(x-y)$, and $F_U$ and $G_V$ be the CDFs of the random variables $U, V$, respectively, where $U$ and $V$ have compact supports. The Wasserstein-like directional metric $d$ is defined as follows:
    \begin{equation} \label{eq:Wasserstein_like_metric}
    \begin{aligned}
        d(F_U,G_V) &\triangleq \inf_{U,V}{(U-V)}\\
        &= \int_0^1 L(F^{-1}_U(q) - G^{-1}_V(q)) dq
    \end{aligned}
\end{equation}
\end{definition}

It is important to emphasize that our directional metric does not measure shape similarity between distributions. As a result, two distributions with similar means but different variances can produce a low value under our directional metric.
Let us understand the above definition with an illustration. Consider two 1D optimal transport scenarios as depicted in Fig. \ref{fig:1-D-OT-illustration}. Here, we aim to quantify the discrepancy between the CDF $F_U(x)$ of a random variable $U$ and the target CDF $G_V(x)$ of a random variable $V$, both defined over the interval $x \in [x_{\min}, x_{\max}]$. 
In both Fig. \ref{fig:1-D-OT-illustration}(a) and \ref{fig:1-D-OT-illustration}(b), although the mass transfer occurs in opposite directions—positive in Fig. \ref{fig:1-D-OT-illustration}(a) and negative in Fig. \ref{fig:1-D-OT-illustration}(b)—the Wasserstein distance assigns the same value in both cases. This metric, therefore, offers no indication of which distribution is preferable. In contrast, the proposed Wasserstein-like directional metric considers the direction of the net flow of mass between distributions, providing insight into their relative superiority. Specifically, in Fig. \ref{fig:1-D-OT-illustration}(a), the net mass transfer from $F_U$ to $G_V$ occurs in the positive direction, signifying that the target distribution $G_V$ is superior. Conversely, in Fig. \ref{fig:1-D-OT-illustration}(b), the net flow of mass is negative, indicating that $F_U$ is the superior distribution. By integrating directional mass flow, this new metric addresses a key shortcoming of the classical Wasserstein distance, offering a more nuanced comparison between distributions.



\begin{figure}[htbp]
    \centering
    \includegraphics[width=0.45\textwidth, keepaspectratio]{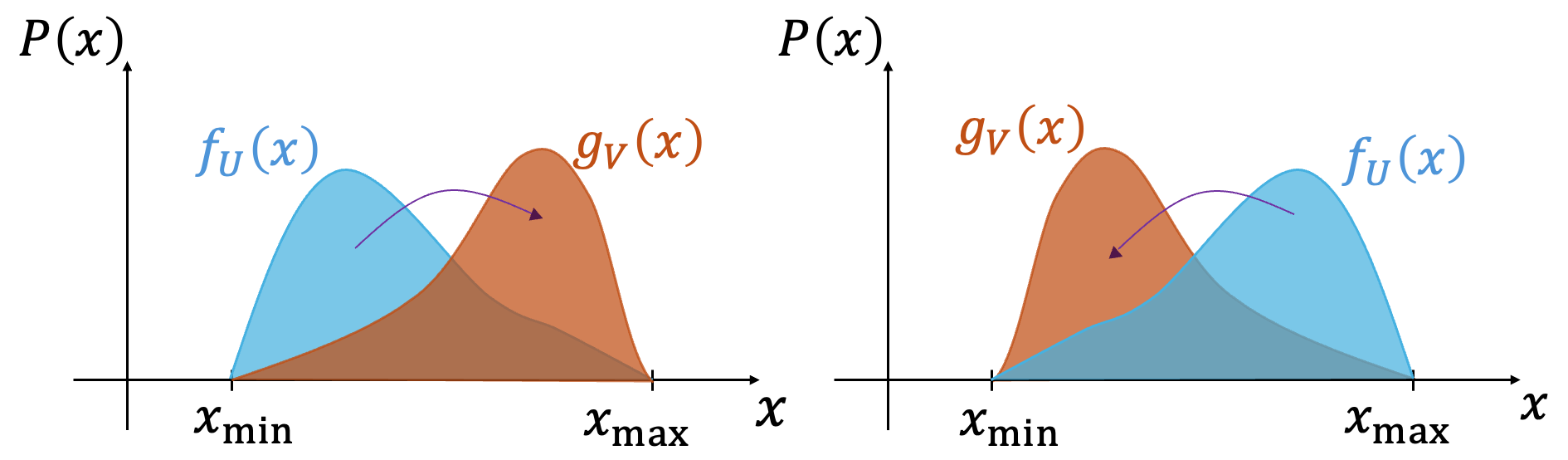}

    \hspace{0.5cm} (a) \hspace{3.5cm} (b)
    
    \caption{Probability density functions $f_U(x)$ and $g_V(x)$ of the two CDF functions $F_U(x)$ and $G_V(x)$ for the random variables $U$ and $V$, respectively, defined on $x \in [x_{\min}, x_{\max}]$. The purple arrow denotes the transfer plan from $F_U$ to $G_V$.}
    \label{fig:1-D-OT-illustration}
\end{figure}

\subsection{Advantage Estimation using Wasserstein-like Directional Metric}


In this section, we will use the proposed Wasserstein-like directional metric (\ref{eq:Wasserstein_like_metric}) to develop a low-bias low-variance {distributional} advantage estimate $\hat{\mathpzc{A}}$ of the advantage function $\mathpzc{A}^{\pi,\gamma}$. We first define the {discounted} distributional TD error $\delta^G(s_t,a_t)$ utilizing the proposed Wasserstein-like directional metric as follows:
\begin{definition}\label{def:distributional_TD_error}
Let $G(\cdot)$ be the value distribution, and $r(\cdot)$ be the reward function. For a given state-action pair $(s_t,a_t)$, the distributional TD error is defined as:
    \begin{equation}\label{eq:distributional_TD_error}
    \begin{aligned}
        \delta^G(s_t,a_t)&\triangleq d\left(r(s_{t},a_{t})+\gamma G(S_{t+1}), G(s_{t})\right)\\
        &= r(s_{t},a_{t})+ d\left(\gamma G(S_{t+1}), G(s_{t})\right)
    \end{aligned}
    \end{equation}
\end{definition}


{Before deriving the distributional GAE (DGAE), we will prove a key lemma that will come in handy later.}

\begin{lemma}\label{lem:lemma1}
    Let $U$ be a random variable with its inverse CDF denoted by $F^{-1}_U$, and let $\eta \in \mathbb{R}$ be a constant scalar. Then, the inverse CDF of the random variable $\eta U$ is given by $\eta F^{-1}_U$. Mathematically, it is $F^{-1}_{\eta U} = \eta F^{-1}_U$.
\end{lemma}

\begin{proof}
Define a new random variable $V = \eta U$. By the definition of the CDF, we have
$F_V(v) = P(V \leq v)
    = P(\eta U \leq v)
    = F_U\left(\frac{v}{\eta}\right).$ Now, recall that for a uniformly distributed variable $q \in [0,1]$, the following holds 
$F_V(F^{-1}_V(q)) =
    F_U \left( \frac{F^{-1}_V(q)}{\eta} \right) = q.$ Consequently, we deduce that
$F^{-1}_V(q) = \eta F^{-1}_U(q).$
Thus, the inverse CDF of the random variable $\eta U$ is $\eta F^{-1}_U$.
\end{proof}




We are now ready to derive the low-bias, low-variance {DGAE} for distributional policy gradient algorithms. If the value distribution $G^{\pi,\gamma}$ is known, the {discounted} distributional TD error $\delta^G(s_t,a_t)$ in Def. \ref{def:distributional_TD_error} will be a {valid}, {unbiased} estimator of the advantage $\mathpzc{A}^{\gamma,\pi}$. However, in most practical scenarios, the true value distribution is unknown and needs to be approximated as $\hat{G}$. In such cases, the bias in the value estimator $\hat{G}$ is transferred to the advantage function, resulting in biased policy gradient estimates \cite{Schulman}.


The bias in the estimator can be reduced by leveraging the observed future rewards and discounting the value distribution in the Bellman update. Hence, we can define a low-bias $n$-step advantage estimator utilizing the known $n$-step trajectory reward sampled using the current policy as follows:
\begin{definition}\label{def:k_step_estimator}
Given a value distribution approximator $\hat{G}$ and an $n$-step trajectory $\tau_{t,n-1}=[s_t, a_t, s_{t+1}, a_{t+1},..., s_{t+n-1}, a_{t+n-1}]$ sampled using the current policy being improved, the $n$-step advantage estimator under the Wasserstein-like directional metric is defined as:
    \begin{equation}\label{eq:k_step_estimator}
    \begin{aligned}
        \hat{\mathpzc{A}}^{(n)}(s_t,a_t) &\triangleq d\left(r(\tau_{t,n-1}) + \gamma^n \hat{G}(S_{t+n}), \hat{G}(s_t) \right) \\
        &= r(\tau_{t,n-1}) + d \left(\gamma^n \hat{G}(S_{t+n}), \hat{G}(s_t) \right)
    \end{aligned}
    \end{equation}
where $r(t,\tau_{n-1}) = \sum_{k=0}^{n-1}\gamma^k r(s_{t+k},a_{t+k})$. The bias in the $n$-step advantage estimator $\hat{\mathpzc{A}}^{(n)}(s_t,a_t)$ gets smaller as $n\rightarrow\infty$ because the term $\gamma^n\hat{G}$ gets heavily discounted.
\end{definition}

\begin{theorem}
    The DGAE $\hat{\mathpzc{A}}^{\gamma,\lambda}_{\text{DGAE}}(s_t,a_t)$ for distributional policy gradient algorithms is defined as the exponentially weighted average of the $n$-step estimators and is given by:
    \begin{align}\label{eq:gae_final}
        \hat{\mathpzc{A}}^{\gamma,\lambda}_{\text{DGAE}}(s_t,a_t) = \sum_{k=0}^{\infty} (\gamma\lambda)^k \delta^{\hat{G}}(s_{t+k},a_{t+k}) 
    \end{align}
    where $\delta^{\hat{G}}(s_{t+k},a_{t+k})$ follows Eq. (\ref{eq:distributional_TD_error}) as the estimated distributional TD error, $\gamma\in[0,1]$ is the discounting factor, and $\lambda\in(0,1)$ is a control parameter. Both $\gamma$ and $\lambda$ are used to trade-off between bias and variance.
\end{theorem}

\begin{proof}
Starting from Def. \ref{def:k_step_estimator} and using the Wasserstein-like directional metric (Def. \ref{def:Wasserstein_like_metric}), we begin with:
\begin{align*}
\hat{\mathpzc{A}}^{(n)}(s_t,a_t) &= \sum_{k=0}^{n-1}\gamma^k r(s_{t+k},a_{t+k}) \\
&+ \int_0^1 L\big(\gamma^n F^{-1}_{\hat{G}(S_{t+n})}(q) - F^{-1}_{\hat{G}(s_t)}(q)\big) dq
\end{align*}
Applying Lemma \ref{lem:lemma1} and strategically adding/subtracting intermediate quantile terms yields:
\begin{align}
\hat{\mathpzc{A}}^{(n)}(s_t,a_t) &= \sum_{k=0}^{n-1}\gamma^k r(s_{t+k},a_{t+k})\nonumber\\
        &\hspace{10mm}+d\bigg(\gamma \hat{G}(S_{t+k+1}),\hat{G}(S_{t+k})\bigg)\nonumber\\
        &= \sum_{k=0}^{n-1} \gamma^k \delta^{\hat{G}}(s_{t+k},a_{t+k})
        \label{eq:n_step_estimator}
\end{align}
We define DGAE $\hat{\mathpzc{A}}^{\gamma,\lambda}_{\text{DGAE}}$ as the exponentially-weighted average of the $n$-step advantage estimators (\ref{eq:n_step_estimator}):
\begin{align*}
    \hat{\mathpzc{A}}^{\gamma,\lambda}_{\text{DGAE}}(s_t,a_t) &\triangleq (1-\lambda)\left(\sum_{k=1}^{\infty}\lambda^{k-1}\hat{\mathpzc{A}}^{(k)}(s_t,a_t)\right)\nonumber\\
    &= (1-\lambda)\left(\sum_{k=0}^{\infty}\gamma^{k}\delta^G(s_{t+k},a_{t+k})\left(\frac{\lambda^k}{1-\lambda}\right)\right)\\
    &= \sum_{k=0}^{\infty} (\gamma\lambda)^k \delta^{\hat{G}}(s_{t+k},a_{t+k}) 
\end{align*}
This concludes the proof
\end{proof}

Eq. (\ref{eq:gae_final}) shows that the {distributional} GAE {for distributional policy gradient methods}
follows a straightforward formulation under the proposed Wasserstein-like directional metric extending GAE to settings with value distributions. 
The parameters $\gamma$ and $\lambda$ 
influence the bias-variance trade-off when using an approximate value distribution: $\gamma\in(0,1)$ introduces bias irrespective of the value distribution accuracy. In contrast, $\lambda\in(0,1)$ introduces bias only when the approximation is inaccurate. Similar to the findings in \cite{Schulman}, the best value for $\lambda$ is always lower than that of $\gamma$.



Integrating {DGAE} with policy gradients is straightforward, naturally extending them to the distributional RL setting. Specifically, we replace $\psi_k$ in Eq. (\ref{eq:gradient_obj_fcn}) with our proposed DGAE $\hat{\mathpzc{A}}^{\gamma,\lambda}_{\text{DGAE}}$. To represent the state-value distribution, we use inverse CDFs, and the value network is trained by minimizing the quantile-Huber loss \cite{Dabney} between the value distribution $F^{-1}_G$ and its Bellman update $F^{-1}_{\mathcal{T} G}$. The quantile-Huber loss is given by:
\begin{equation}\label{eq:quantile_huber_loss}
    \rho^{\kappa}_{q}(u) = |q - \delta_{\{u<0\}}| \mathcal{L}_{\kappa}(u)
\end{equation}
where {$u=F^{-1}_{\mathcal{T} G}-F^{-1}_G$,} 
$\kappa$ is a constant, $\delta_{\{u<0\}}$ is the indicator function, and
\begin{equation}
    \mathcal{L}_{\kappa}(u) =
    \begin{cases}
        \frac{1}{2}u^2, & \text{if } |u| \leq \kappa, \\
        \kappa\left(|u| - \frac{1}{2}\kappa\right), & \text{otherwise}. \nonumber
    \end{cases}
\end{equation}

The overall approach is detailed in Algorithm \ref{alg:policy_value_update}.

\begin{algorithm}
\caption{Policy and Value Distribution Update}\label{alg:policy_value_update}

\begin{algorithmic}[1]
    \STATE Initialize policy network $\pi_{\theta_0}$ and value distribution network ${G}_{\phi_0}$, where $\theta$ and $\phi$ are the network parameters.
    \FOR{$i=0,1,2,...$}
        \STATE Sample $n$ time steps $\{s_t, a_t, r_t, s_{t+1}\}$ using the current policy $\pi_{\theta_i}$
        \STATE Compute $\delta^{G_{\phi_i}}(s_t,a_t)$ for all time steps
        \STATE Compute advantage estimates $\hat{\mathpzc{A}}^{\gamma, \lambda}_{\text{DGAE}}(s_t, a_t)$ for all time steps:
             \[
                 \hat{\mathpzc{A}}^{\gamma, \lambda}_{\text{DGAE}}(s_t,a_t) = \sum_{k=0}^{n-1} (\gamma \lambda)^k \delta^{G_{\phi_i}}(s_{t+k},a_{t+k})
             \] 
        \STATE Use $\hat{\mathpzc{A}}^{\gamma, \lambda}_{\text{DGAE}}$ to update policy parameters $\theta_i$ using policy gradient approach:
        \[
            \theta_{i+1} \leftarrow \theta_i + \alpha \nabla_{\theta_i} J(\pi_{\theta_i})
        \]
        where $\alpha{>0}$ is the learning rate.
        \STATE Compute value loss $\mathcal{L}_{G_{\phi_i}}$ using Eq. (\ref{eq:quantile_huber_loss}):
        \[
        \mathcal{L}_{G_{\phi_i}} = \sum_{q_i}\mathbb{E}_{q_j}[\rho^{\kappa}_{q_i}(\mathcal{T}G_{\phi_i}(s_{t},q_j)-G_{\phi_i}(s_t,q_i))]
        \] 
        \STATE Update value network parameters $\phi$ using gradient descent:
        \[
            \phi_{i+1} \leftarrow \phi_i - \alpha \nabla_{\phi} \mathcal{L}_{G_{\phi_i}}
        \]
    \ENDFOR
\end{algorithmic}
\end{algorithm}


\section{Simulation Results}\label{sec:results}
We evaluate our approach by integrating the proposed DGAE with proximal policy optimization (PPO) \cite{Schulman3}, trust-region policy optimization (TRPO) \cite{Schulman2}, and asynchronous actor-critic (A2C) \cite{Mnih}, extending them to distributional variants (DPPO, DTRPO, DA2C). We limit DGAE to on-policy policy gradient methods here, as widely adopted off-policy approaches such as the deep deterministic policy gradient (DDPG), twin delayed DDPG (TD3), and soft actor critic (SAC) bypass advantage computation, instead relying on Q-value estimates for policy updates. Although off-policy methods like the actor critic with experience replay (ACER) and importance weighted actor learner architecture (IMPALA) employ advantage functions, their performance often lags behind modern baselines in terms of stability and scalability.

Experiments are conducted on four OpenAI Gym MuJoCo environments: Ant, Hopper, and Swimmer, Walker2d (Fig. \ref{img:gym_env_images}). Each task requires learning continuous control policies to maximize forward velocity. The agent observes a state vector (joint positions/velocities) and outputs joint torques, with rewards reflecting movement efficiency.


\begin{figure}[htbp] 
    \centering
    \includegraphics[width=0.45\textwidth, keepaspectratio]{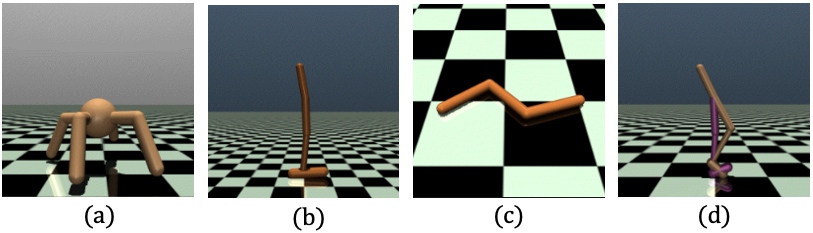}
    \caption{OpenAI gym robot models used for testing: (a) Ant-v3; (b) Hopper-v3; (c) Swimmer-v3; (d) Walker2d-v3}
    \label{img:gym_env_images}
\end{figure}

\subsubsection{Experimental Setup}
We trained baseline algorithms (PPO, TRPO, A2C) and their distributional variants (DPPO, DTRPO, DA2C) for $10^7$ timesteps on Ant and Hopper environments, and $5 \times 10^6$ timesteps on Swimmer and Walker2d. Each configuration used $5$ random seeds to ensure statistical robustness, with results (mean ± std. dev. of undiscounted returns) shown in Fig. \ref{fig:results}. To ensure equitable algorithmic comparisons, we systematically optimized the bias-variance tradeoff parameters ($\gamma$ and $\lambda$) for each method. Fig. \ref{fig:results} contrasts their performance under these individually tuned configurations.


We use fully connected networks with 2 hidden layers to represent the value and policy networks. For most environments, the hidden layers consist of 256 units. However, in the Ant environment, the increased complexity of the state-action space requires expanding the hidden layers to 512 units per layer for improved performance. The policy network outputs the Gaussian policy parameters, i.e., the mean and covariance. The value network output layers differ among implementations: baseline versions produce a single scalar value estimate, while distributed implementations output $64$ quantiles to facilitate distributional value representation. This ensures fair comparison by isolating the impact of distributional value estimation from other structural variations.

\subsubsection{Training Results}

\begin{figure*}[t]
    \centering
    \begin{subfigure}{0.24\textwidth}
        \centering
        \includegraphics[width=\textwidth]{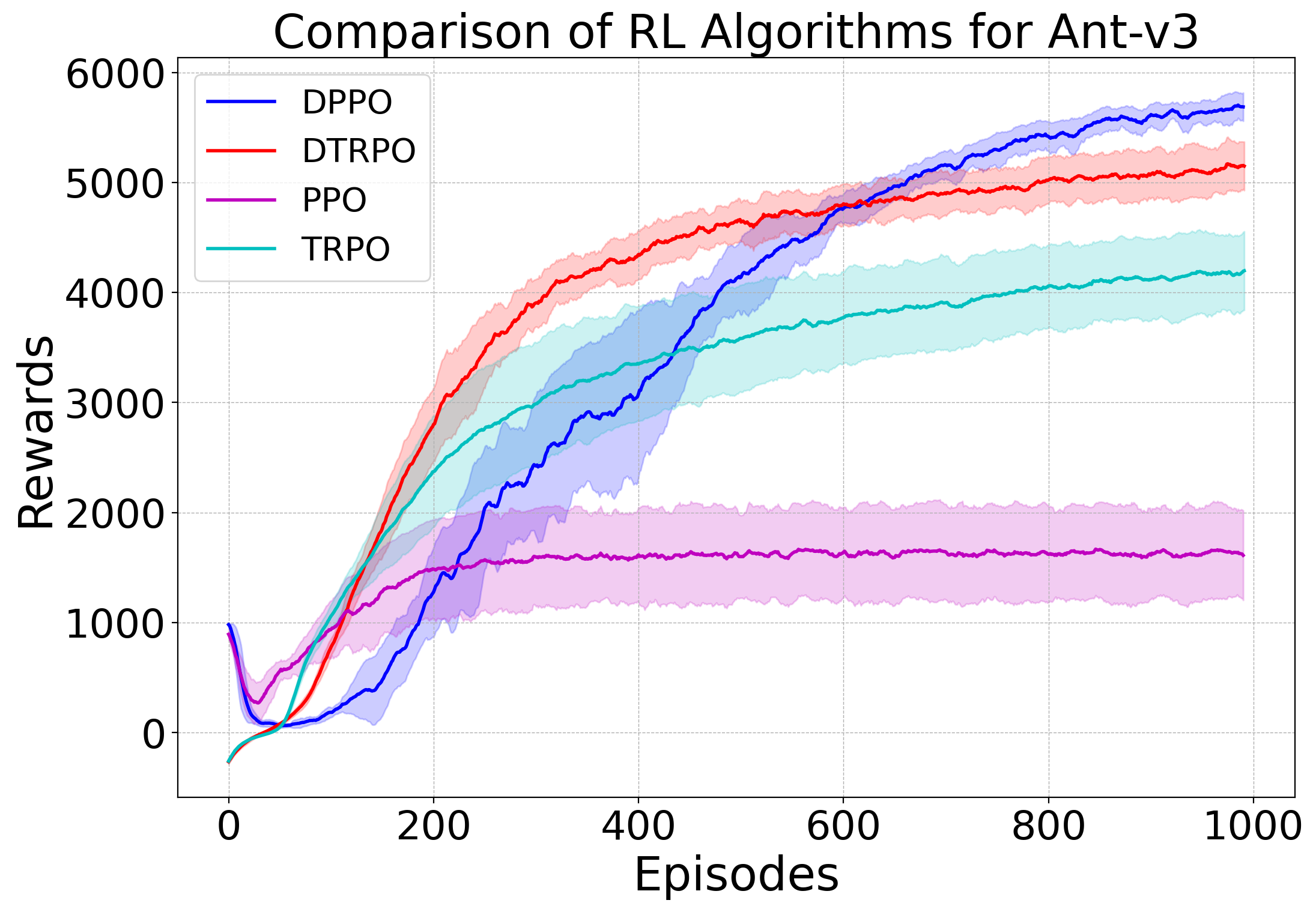}
        \caption{Ant-v3}
        \label{img:Ant_results}
    \end{subfigure}
    \hfill
    \begin{subfigure}{0.24\textwidth}
        \centering
        \includegraphics[width=\textwidth]{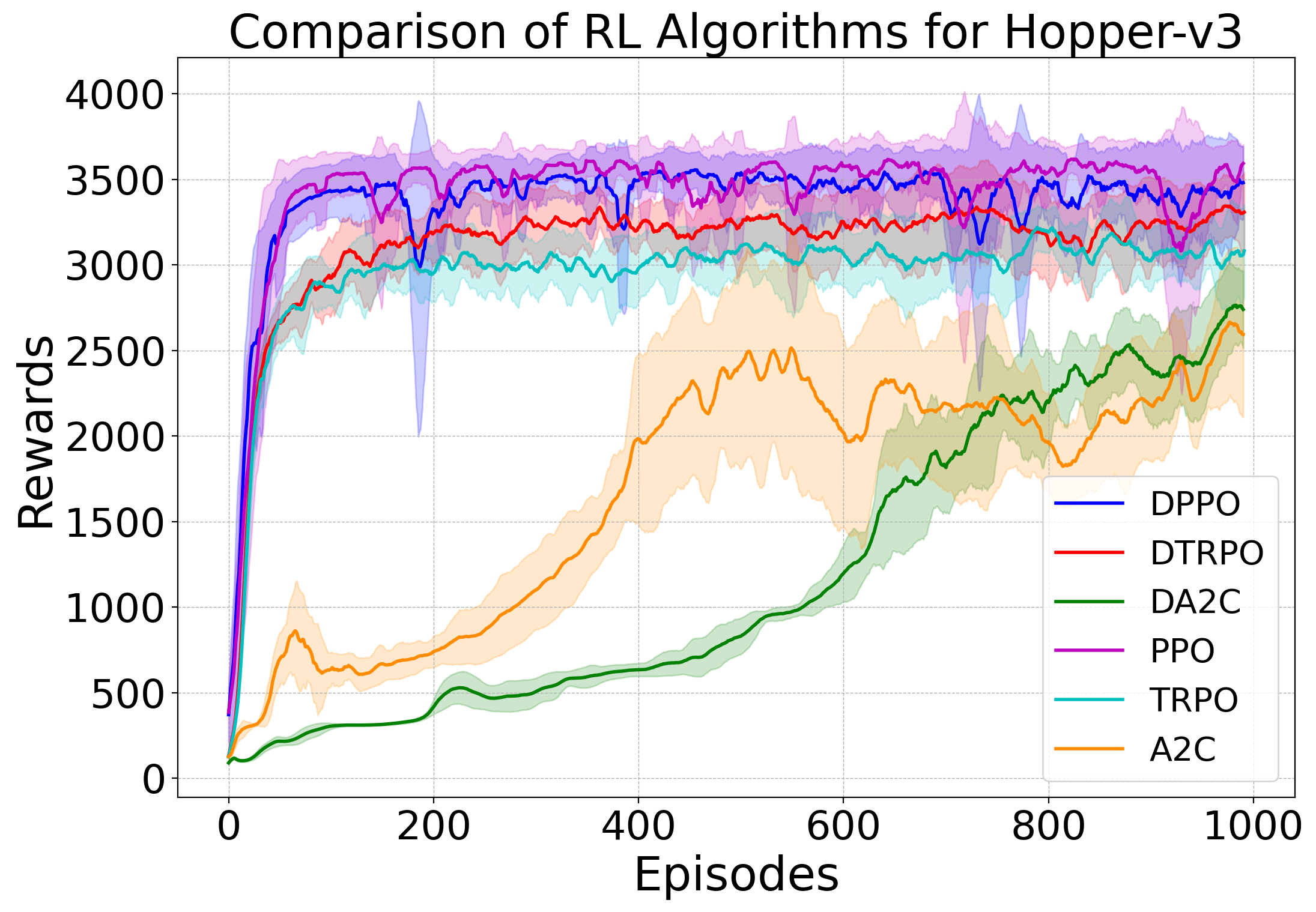}
        \caption{Hopper-v3}
        \label{img:Hopper_results}
    \end{subfigure}
    \hfill
    \begin{subfigure}{0.24\textwidth}
        \centering
        \includegraphics[width=\textwidth]{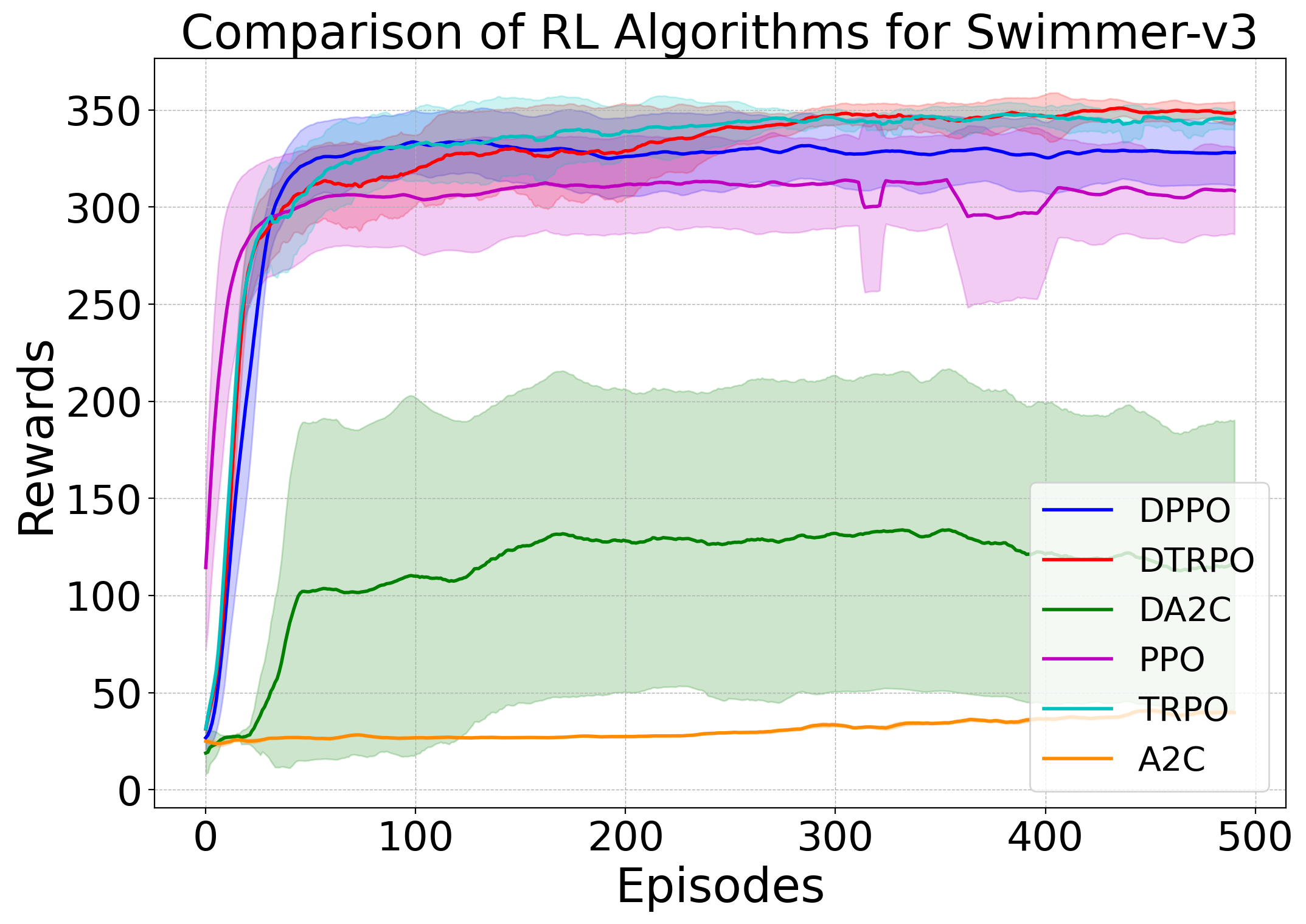}
        \caption{Swimmer-v3}
        \label{img:Swimmer_results}
    \end{subfigure}    
    \hfill
    \begin{subfigure}{0.24\textwidth}
        \centering
        \includegraphics[width=\textwidth]{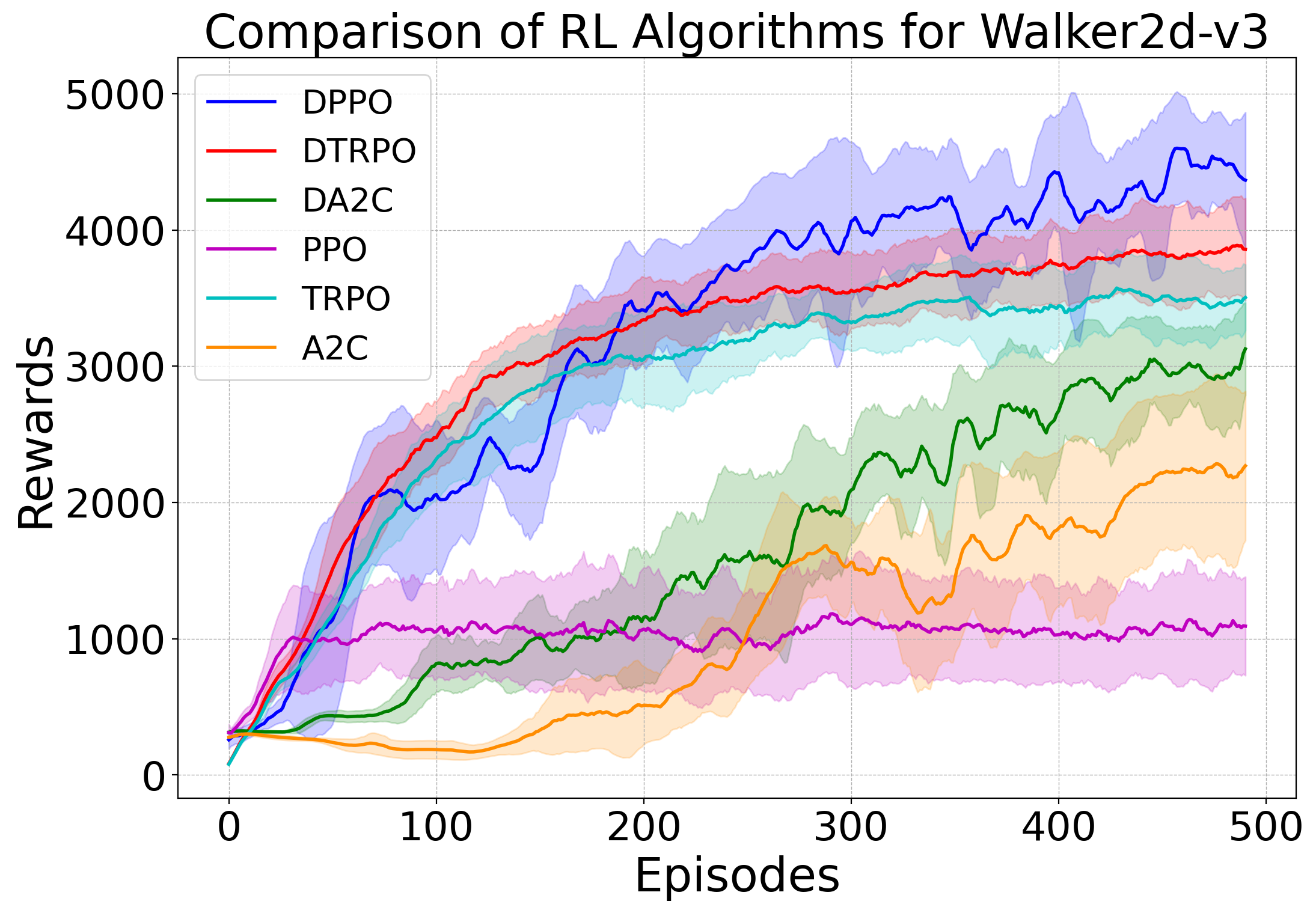}
        \caption{Walker2d-v3}
        \label{img:Walker_results}
    \end{subfigure}
    \caption{Learning curves on OpenAI gym environments with solid line representing the mean and shaded region representing the standard deviation of the undiscounted return across 5 training epochs.}
    \label{fig:results}
\end{figure*}

We evaluated the performance of the algorithms by plotting the mean and standard deviation of the undiscounted return across 5 training epochs, as shown in Fig. \ref{fig:results}. Fig. \ref{img:Ant_results} presents the results across algorithms in the Ant environment. The DPPO and PPO algorithms have a rollout buffer length set to 4,096 for the Ant environment, whereas the TRPO and DTRPO have a rollout buffer size of 50,000. In this setting, the DPPO algorithm significantly outperforms the baseline PPO, exhibiting a large performance gap. Similarly, DTRPO also exhibits a considerable performance improvement over the baseline TRPO. We were unable to achieve stable performance for A2C and DA2C in this environment. Notably, A2C and DA2C showed consistently poor performance across all environments. For the rest of the environments, the rollout buffer lengths for (D)PPO and (D)TRPO were set as 2048 and 25000, respectively. In the Hopper environment (Fig. \ref{img:Hopper_results}), both distributional variants, DPPO and DTRPO, surpass their baseline counterparts. DA2C exhibits a slight improvement over A2C, though the overall performance of both remains relatively weak compared to the other algorithms. Fig. \ref{img:Swimmer_results} shows the results for the Swimmer environment. Here, DTRPO and TRPO display very similar performance and outperform the other algorithms. DPPO performs better than PPO, and DA2C outperforms A2C. Lastly, Fig. \ref{img:Walker_results} presents the results for the Walker environment. DPPO achieves the best results by a significant margin, followed by DTRPO. Consistent with previous environments, DA2C outperforms A2C, but PPO fails to deliver competitive results in this particular environment. As mentioned previously, distributions with similar means yield low advantage values, regardless of their variances. Nevertheless, these scenarios rarely arise in practice. For example, in the Hopper environment, this situation occurred only in about $0.093\%$ of the instances.

Overall, integrated with our proposed DGAE algorithm, the distributional policy gradient algorithms outperform or at least perform on par with their traditional baselines. Since we used the same rollout length across the traditional and distributional variants, we can confirm that our approach has sampling efficiency similar to traditional baseline algorithms. Moreover, the parameters $\gamma, \lambda$ have similar effects on the bias-variance trade-off as in traditional GAE.

\section{Conclusion}\label{sec:conclusion}

In this work, we introduced a novel metric based on optimal transport theory, termed the Wasserstein-like directional metric, to quantify the distance and superiority between distributions. Leveraging this metric, we developed a DGAE tailored for distributional policy gradient algorithms. The DGAE seamlessly extends traditional policy gradient methods to the distributional RL setting by modifying the value functions into inverse CDFs.

Our results demonstrate that integrating the proposed DGAE with baseline policy gradient algorithms enhances their performance in most cases. Additionally, the sampling efficiency of DGAE, when combined with policy gradients, closely mirrors the sampling efficiency of traditional GAE used with baseline algorithms. We also observed that the bias-variance tradeoff parameters, $\gamma$ and $\lambda$, in DGAE behave similarly to their role in traditional GAE, maintaining a comparable impact on the tradeoff.


\bibliographystyle{plain} 
\bibliography{references.bib} 

\begin{thebibliography}{10}

\bibitem{Arjovsky}
Martin Arjovsky, Soumith Chintala, and L{\'e}on Bottou.
\newblock Wasserstein generative adversarial networks.
\newblock In {\em International conference on machine learning}, pages 214--223. PMLR, 2017.

\bibitem{Bellemare}
Marc~G Bellemare, Will Dabney, and R{\'e}mi Munos.
\newblock A distributional perspective on reinforcement learning.
\newblock In {\em International conference on machine learning}, pages 449--458. PMLR, 2017.

\bibitem{Dabney}
Will Dabney, Mark Rowland, Marc Bellemare, and R{\'e}mi Munos.
\newblock Distributional reinforcement learning with quantile regression.
\newblock In {\em Proceedings of the AAAI conference on artificial intelligence}, volume~32, 2018.

\bibitem{Haarnoja}
Tuomas Haarnoja, Aurick Zhou, Kristian Hartikainen, George Tucker, Sehoon Ha, Jie Tan, Vikash Kumar, Henry Zhu, Abhishek Gupta, Pieter Abbeel, et~al.
\newblock Soft actor-critic algorithms and applications.
\newblock {\em arXiv preprint arXiv:1812.05905}, 2018.

\bibitem{Li}
Luchen Li and A.~Aldo Faisal.
\newblock Bayesian distributional policy gradients, 2021.

\bibitem{Lillicrap}
Timothy~P Lillicrap, Jonathan~J Hunt, Alexander Pritzel, Nicolas Heess, Tom Erez, Yuval Tassa, David Silver, and Daan Wierstra.
\newblock Continuous control with deep reinforcement learning.
\newblock {\em arXiv preprint arXiv:1509.02971}, 2015.

\bibitem{Mnih}
Volodymyr Mnih, Adria~Puigdomenech Badia, Mehdi Mirza, Alex Graves, Timothy Lillicrap, Tim Harley, David Silver, and Koray Kavukcuoglu.
\newblock Asynchronous methods for deep reinforcement learning.
\newblock In {\em International conference on machine learning}, pages 1928--1937. PmLR, 2016.

\bibitem{Santambrogio}
Filippo Santambrogio.
\newblock Optimal transport for applied mathematicians.
\newblock 2015.

\bibitem{Schulman2}
John Schulman, Sergey Levine, Pieter Abbeel, Michael Jordan, and Philipp Moritz.
\newblock Trust region policy optimization.
\newblock In {\em International conference on machine learning}, pages 1889--1897. PMLR, 2015.

\bibitem{Schulman}
John Schulman, Philipp Moritz, Sergey Levine, Michael Jordan, and Pieter Abbeel.
\newblock High-dimensional continuous control using generalized advantage estimation.
\newblock {\em arXiv preprint arXiv:1506.02438}, 2015.

\bibitem{Schulman3}
John Schulman, Filip Wolski, Prafulla Dhariwal, Alec Radford, and Oleg Klimov.
\newblock Proximal policy optimization algorithms.
\newblock {\em arXiv preprint arXiv:1707.06347}, 2017.

\bibitem{Sutton}
Richard~S Sutton, Andrew~G Barto, et~al.
\newblock {\em Reinforcement learning: An introduction}, volume~1.
\newblock MIT press Cambridge, 1998.

\bibitem{Sutton2}
Richard~S Sutton, David McAllester, Satinder Singh, and Yishay Mansour.
\newblock Policy gradient methods for reinforcement learning with function approximation.
\newblock {\em Advances in neural information processing systems}, 12, 1999.

\end{thebibliography}

\end{document}